\documentclass[sigconf]{acmart}

\usepackage{amsthm}
\usepackage{amsmath}
\usepackage{algorithm}
\usepackage{enumitem}
\usepackage{picinpar}
\usepackage{lineno}
\usepackage{graphicx}

\usepackage{algorithmicx}
\usepackage[noend]{algpseudocode}
\usepackage{subfigure}
\usepackage{multirow}
\usepackage{color}
\usepackage{balance}
\usepackage{enumitem}
\usepackage{hhline}
\usepackage[normalem]{ulem}
\usepackage{booktabs}
\usepackage{wrapfig}
\usepackage{cancel}
\usepackage{hyperref}
\usepackage{makecell}

\newtheorem{theorem}{Theorem}

\newtheorem{definition}{Definition}

\newcommand{\bL}{\ensuremath{\mathcal{L}}}

\newcommand{\bN}{\ensuremath{\mathcal{N}}}

\renewcommand{\vec}[1]{\ensuremath{\mathbf{#1}}}

\newcommand{\stitle}[1]{\vspace{1mm} \noindent {\bf #1}}

\newcommand{\ie}{{\it i.e.}}

\newcommand{\method}[1]{\textsc{#1}}
\newcommand{\model}{\method{ProNoG}{}}

\newcommand{\eat}[1]{}

\newcommand{\stkout}[1]{\ifmmode\text{\sout{\ensuremath{#1}}}\else\sout{#1}\fi}

\AtBeginDocument{%
  }

\copyrightyear{2025}
\acmYear{2025}
\setcopyright{acmlicensed}
\acmConference[KDD '25] {Proceedings of the 31st ACM SIGKDD Conference on Knowledge Discovery and Data Mining V.1}{August 3--7, 2025}{Toronto, ON, Canada.}
\acmBooktitle{Proceedings of the 31st ACM SIGKDD Conference on Knowledge Discovery and Data Mining V.1 (KDD '25), August 3--7, 2025, Toronto, ON, Canada}
\acmDOI{10.1145/3690624.3709219}
\acmISBN{979-8-4007-1245-6/25/08}

\begin{document}

\title{Non-Homophilic Graph Pre-Training and Prompt Learning}

\author{Xingtong Yu$^{*\ddagger}$}
\affiliation{%
 \institution{Singapore Management University}
  \country{Singapore}}
\email{xingtongyu@smu.edu.sg}

\author{Jie Zhang$^*$}
\affiliation{%
 \institution{National University of Singapore}
  \country{Singapore}}
\email{jiezhang_jz@u.nus.edu}

\author{Yuan Fang$^{\dagger}$}
\affiliation{%
  \institution{Singapore Management University}
  \country{Singapore}}
\email{yfang@smu.edu.sg}

\author{Renhe Jiang$^{\dagger}$}
\affiliation{%
  \institution{The University of Tokyo}
  \country{Japan}}
\email{jiangrh@csis.u-tokyo.ac.jp}

\thanks{
    $^*$Co-first authors. $^{\ddagger}$Part of the work was done at the University of Tokyo.\\
    $^{\dagger}$Corresponding authors.
}


\renewcommand{\shortauthors}{Xingtong Yu, Jie Zhang, Yuan Fang, and Renhe Jiang}

\begin{abstract}
Graphs are ubiquitous for modeling complex relationships between objects across various fields. Graph neural networks (GNNs) have become a mainstream technique for graph-based applications, but their performance heavily relies on abundant labeled data. To reduce labeling requirement, pre-training and 
prompt learning has become a popular alternative. However, most existing prompt methods do not distinguish between homophilic and heterophilic characteristics in graphs.  
In particular, many real-world graphs are \emph{non-homophilic}---neither strictly nor uniformly homophilic---as they exhibit varying homophilic and heterophilic patterns across graphs and nodes.
In this paper, we propose \model, a novel pre-training and prompt learning framework for such non-homophilic graphs. First, we examine existing graph pre-training methods, providing insights into the choice of pre-training tasks. Second, recognizing that each node exhibits unique non-homophilic characteristics, we propose a conditional network to characterize node-specific patterns in downstream tasks. Finally, we thoroughly evaluate and analyze \model\ through extensive experiments on ten public datasets. 
\end{abstract}





\maketitle

\section{Introduction}
Graph data are pervasive in real-world applications, such as citation networks
, social networks
, transportation systems, and molecular graphs
. Traditional methods typically train graph neural networks (GNNs) \cite{kipf2016semi,velivckovic2017graph} or graph transformers \cite{ying2021transformers,yun2019graph} in a supervised manner. However, they require substantial labeled data and re-training for each specific task.

To mitigate the limitations of supervised methods, pre-training methods have gained significant traction \cite{hu2020gpt,velickovic2019deep,you2020graph}. They first learn universal, task-independent properties from unlabeled graphs, and then fine-tune the pre-trained models to various downstream tasks using task-specific labels \cite{velickovic2019deep,you2020graph}. However, a significant gap occurs between the pre-training objectives and downstream tasks, resulting in suboptimal performance \cite{yu2024few,sun2023graph}. Moreover, fine-tuning large pre-trained models is costly and still requires sufficient task-specific labels to prevent overfitting.
As an alternative to fine-tuning, prompt learning has emerged as a popular parameter-efficient technique for adaptation to downstream tasks \cite{sun2022gppt,liu2023graphprompt,fang2024universal,yu2024generalized}. They first utilize a universal template to unify pre-training and downstream tasks. 
Then, a learnable prompt is employed to modify the input features or hidden embeddings of the pre-trained model to align with the downstream task without updating the pre-trained weights. Since a prompt has far fewer parameters than the pre-trained model, prompt learning can be especially effective in low-resource settings \cite{yu2024few}.

However, current graph ``pre-train, prompt'' approaches rely on the homophily assumption 
or overlook the presence of heterophilic edges. 
Specifically, the homophily assumption \cite{ma2021homophily,zhu2020beyond} states that neighboring nodes 
should share the same labels, whereas heterophily refers to the opposite scenario where two neighboring nodes have different labels. We observe that real-world graphs are typically \emph{non-homophilic}, meaning they are \emph{neither strictly nor uniformly homophilic} and mix \emph{both homophilic and heterophilic patterns} \cite{xiao2022decoupled,xiao2024simple}.
In this work, we investigate the pre-training and prompt learning methodology for non-homophilic graphs. We first revisit existing graph pre-training methods, and then propose a \textbf{Pro}mpt learning framework for \textbf{N}on-h\textbf{o}mophilic \textbf{G}raphs (or \model\ in short). The solution is non-trivial,
as the notion of homophily encompasses two key aspects, each with its own unique challenge.  

First, different graphs exhibit varying degrees of non-homophily.
As shown in Fig.~\ref{fig.intro-motivation}(a), the \textit{Cora} citation network is typically considered largely homophilic with 81\% homophilic edges\footnote{Defined as edges connecting two nodes of the same label; see Eq.~\eqref{eq.graph-level-edge-homophily-ratio} in Sect.~\ref{sec.preliminaries}.}, whereas the \textit{Wisconsin} web graph links different kinds of webpages, which is highly heterophilic with only 21\% homophilic edges.  
Moreover, the non-homophilic characteristics of a graph also depends on the target label. For example, in a dating network shown in Fig.~\ref{fig.intro-motivation}(b), taking gender as the node label, the graph is more heterophilic with $2/7$ homophilic edges. 
However, taking hobbies as the node label, the graph becomes more homophilic with $4/7$ homophilic edges. 
Hence, \textit{how do we pre-train a graph model irrespective of the graph's homophily characteristics?} 
In this work, we propose definitions for \emph{homophily tasks} and \emph{homophily samples}. We show that 
pre-training with non-homophily samples increases the loss of any homophily task. Meanwhile, a less homophilic graph results in a higher number of non-homophily samples, subsequently increasing the pre-training loss for homophily tasks. This motivates us to move away from homophily tasks for graph pre-training \cite{sun2022gppt,liu2023graphprompt} and instead choose a non-homophily task \cite{you2020graph,xiao2022decoupled}. 

\begin{figure}[t]
\centering
\includegraphics[width=1\linewidth]{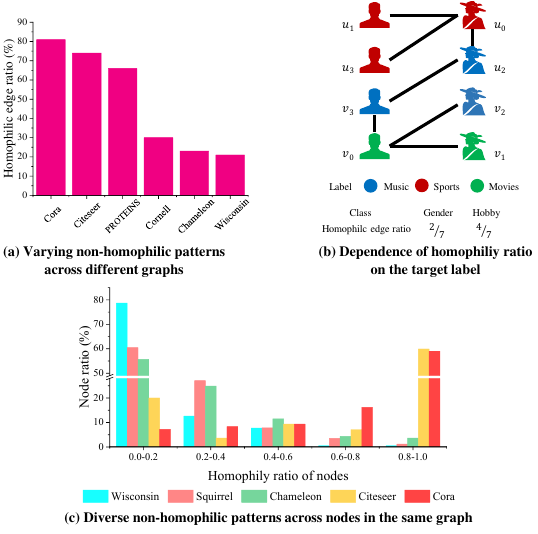}
\vspace{-6mm}
\caption{Non-homophilic characteristics of graphs.}
\label{fig.intro-motivation}
\end{figure}

Second, different nodes within the same graph are distributed differently in terms of their non-homophilic characteristics. As shown in Fig.~\ref{fig.intro-motivation}(c), 
on each dataset,
different nodes within the same graph generally exhibit diverse homophily ratios\footnote{Defined as the proportion of a node's neighbors that share the same label as the node; refer to Eq.~\eqref{eq.node-homophily-ratio} in Sect.~\ref{sec.preliminaries}.}. Hence, \textit{how do we capture the fine-grained, node-specific non-homophilic characteristics?} 
Due to the diverse characteristics across nodes, a one-size-fits-all solution for all nodes would be inadequate. 
However, existing approaches generally apply a single prompt to all nodes \cite{sun2022gppt, liu2023graphprompt}, treating all nodes uniformly. Thus, these methods overlook the fine-grained node-wise non-homophilic characteristics, leading to suboptimal performance. 
Although some recent works \cite{tan2023virtual, chen2023ultra} have proposed node-specific prompts, they are not designed to account for the variation in nodes' non-homophilic charactersitics. Inspired by conditional prompt learning \cite{zhou2022conditional}, we propose generating a unique prompt from each node with a conditional network (condition-net)  to capture the distinct characteristics of each node. We first capture the non-homophilic patterns of each node by reading out its multi-hop neighborhood. Then, conditioned on these non-homophilic patterns, the condition-net produces a series of prompts, one for each node that reflects its varying non-homophilic characteristics. These prompts can adjust the node embeddings to better align them with the downstream task.

In summary, the contributions of this work are threefold:
(1) We observe varying degrees of homophily across graphs, which motivates us to revisit graph pre-training tasks. We provide theoretical insights which guide us to choose non-homophily tasks for graph pre-training. 
(2) We further observe that, within the same graph, different nodes have diverse distributions of non-homophilic characteristics. To adapt to the unique non-homophilic patterns of each node, we propose the \model\ framework for non-homophilic prompt learning, which is equipped with a condition-net to generate a series of prompts conditioned on each node. The node-specific prompts enables fine-grained, node-wise adaptation for the downstream tasks.
(3) We perform extensive experiments on ten benchmark datasets, demonstrating the superior performance of \model\ compared to a suite of state-of-the-art methods.
\section{Related Work}

In the following, we briefly review the literature on general and non-homophilic graph learning, as well as graph prompt learning.

\stitle{Graph representation learning.}
GNNs \cite{kipf2016semi, velivckovic2017graph} are mainstream approaches for graph representation learning. They typically operate on a message-passing framework, where nodes iteratively update their representations by aggregating messages received from their neighboring nodes \cite{xu2018powerful, hamilton2017inductive, yu2023learning}. However, their effectiveness relies on abundant task-specific labeled data and requires re-training for each task.
Inspired by the success of the pre-training paradigm in the language \cite{dong2019unified, brown2020language, gao2021making, schick2021s} and vision \cite{bao2022beit, zhou2022learning, zhou2022conditional, zang2022unified} domains, pre-training methods \cite{kipf2016variational, hu2020strategies, hu2020gpt, lu2021learning} have been widely explored for graphs. These methods first pre-train a graph encoder based on self-supervised tasks, and subsequently transfer the pre-trained prior knowledge to downstream tasks. However, these pre-training methods often make the homophily assumption, overlooking that real-world graphs are generally non-homophilic.

\stitle{Non-homophilic graph learning.}
Many GNNs \cite{ma2021homophily, zhu2021graphaaai, luan2022revisiting} have been proposed for non-homophilic graphs, employing techniques such as capturing high-frequency signals \cite{bo2021beyond}, discovering potential neighbors \cite{pei2020geom, jin2021node}, and high-order message passing \cite{zhu2020beyond}. Moreover, recent works have explored pre-training on non-homophilic graphs \cite{xiao2022decoupled, he2023contrastive, xiao2024simple} by capturing neighborhood information to construct unsupervised tasks for pre-training a graph encoder, and then transferring prior non-homophilic knowledge to downstream tasks through fine-tuning with task-specific supervision.
However, a significant gap exists between the objectives of pre-training and fine-tuning \cite{liu2023pre, yu2024few, sun2023graph}. While pre-training focuses on learning inherent graph properties without supervision, fine-tuning adapts these properties to downstream tasks based on task-specific supervision. This discrepancy hinders effective knowledge transfer and negatively impacts downstream performance.

\stitle{Graph prompt learning.}
Originally developed for the language domain, prompt learning effectively unifies pre-training and downstream objectives \cite{brown2020language,liu2021gpt,lester2021power}. Recently, graph prompt learning has emerged as a popular alternation to fine-tuning methods \cite{liu2023graphprompt,sun2022gppt,yu2023hgprompt,yu2024generalized}. These methods first propose a unified template, then design prompts specifically tailored to each downstream task, allowing them to better align with the pre-trained model while keeping the pre-trained parameters frozen. However, current graph prompt learning methods  \cite{yu2024few,sun2023all} typically do not consider the fact that real-world graphs are generally non-homophilic, exhibiting a mixture of diverse homophilic and heterophilic patterns across nodes. Hence, these methods usually apply a single prompt to all nodes, overlooking the unique non-homophilic pattern of each node.

\section{Preliminaries}\label{sec.preliminaries}
\stitle{Graph.}
A graph is defined as \( G = (V, E) \), where \( V \) represents the set of nodes and \( E \) represents the set of edges. The nodes are also associated with a feature matrix $\mathbf{X} \in \mathbb{R}^{|V| \times d}$, such that \( \vec{x}_v \in \mathbb{R}^d \) is a row of $\mathbf{X}$ representing the feature vector for node \( v \in V \). For a collection of multiple graphs, we use the notation \( \mathcal{G} = \{ G_1, G_2, \dots, G_N \} \).

\stitle{Homophily ratio.} 
Given a mapping between the nodes of a graph and a predefined set of labels, let $y_v$ denote the label mapped to node $v$. The homophily ratio $\mathcal{H}(G)$ evaluates the relationships between the labels and the graph structure \cite{zhu2020beyond,ma2021homophily}, measuring the fraction of homophilic edges whose two end nodes share the same label. More concretely,
\begin{align}\label{eq.graph-level-edge-homophily-ratio}
    \mathcal{H}(G) = \frac{\left|\{(u,v) \in E : y_u = y_v\}\right|}{|E|}.
\end{align}
Additionally, the homophily ratio can be defined for each node based on its local structure \cite{mao2024demystifying,xiao2024simple}, measuring the fraction of a node's neighbors that share the same label. 
This node-specific ratio can be defined as
\begin{align}\label{eq.node-homophily-ratio}
    \mathcal{H}(v) = \frac{|\{u \in \mathcal{N}(v) : y_u = y_v\}|}{|\mathcal{N}(v)|},
\end{align}
where $\mathcal{N}(v)$ is the set of neighboring nodes of $v$. 
Note that both $\mathcal{H}(G)$ and $\mathcal{H}(v)$ fall in $[0,1]$. 
Graphs or nodes with a larger proportion of homophilic edges have a higher homophily ratio.

\stitle{Graph encoder.}
Graph encoders learn latent representations of graphs, embedding their nodes into some feature space. A widely used family of graph encoders is GNNs, which typically utilize a message-passing mechanism \cite{wu2020comprehensive,zhou2020graph}. Specifically, each node aggregates messages from its neighbors to generate its own representation. By stacking multiple layers, GNNs enables recursive message passing throughout the graph. Formally, the embedding of a node \( v \) in the \( l \)-th GNN layer, denoted as \( \vec{h}^l_v \), is computed as follows.
\begin{align}
    \vec{h}^l_v = \mathtt{Aggr}(\vec{h}^{l-1}_v, \{\vec{h}^{l-1}_u : u \in \bN(v)\}; \theta^l),
\end{align}
where 
\( \theta^l \) are the learnable parameters in the \( l \)-th layer, and \( \mathtt{Aggr}(\cdot) \) is the aggregation function, which can take various forms \cite{hamilton2017inductive,kipf2016semi,velivckovic2017graph,xu2018powerful,yu2023learning}. In the first layer, the input node embedding \( \vec{h}^0_v \) is typically initialized from the node feature vector \( \vec{x}_v \). The full set of learnable parameters is denoted as \( \Theta = \{\theta^1, \theta^2, \ldots\} \). For simplicity, we define the output node representations of the final layer as \( \vec{h}_v \), which can then be fed into the loss function for a specific task.

\stitle{Problem statement.}
In this work, we aim to pre-train a graph encoder and develop a prompt learning framework for non-homophilic graphs. More specifically, both the pre-training and prompt learning are not sensitive to the diverse non-homophilic characteristics of the graph and its nodes. 

To evaluate our non-homophilic pre-training and prompt learning, we focus on two common tasks on graph data: node classification and graph classification, in few-shot settings.
For node classification within a graph \( G = (V, E) \) over a set of node classes  \( Y \), each node \( v_i \in V \) has a class label \( y_i \in Y \). Similarly, for graph classification across a graph collection \( \mathcal{G} \) with class labels \( Y \), each graph \( G_i \in \mathcal{G} \) has a class label \( Y_i \in Y \). In the few-shot setting, there are only \( k \) labeled samples per class, where \( k \) is a small number (e.g., \( k \le 10 \)). This scenario is known as \( k \)-shot classification \cite{liu2023graphprompt,yu2024generalized,yu2023multigprompt}. Note that the homophily ratio is defined with respect to some predefined set of labels, which may or may not be related to the class labels in downstream tasks.
\section{Revisiting Graph Pre-training}\label{sec.theory}
In this section, we revisit graph pre-training tasks to cope with non-homophilic graphs. We first propose the definition of \textit{homophily tasks} and reveal its connection to the training loss. The theoretical insights further guide us in choosing graph pre-training tasks.

\subsection{Theoretical Insights}

We focus on contrastive graph pre-training tasks. Consider a mainstream contrastive task \cite{you2020graph,qiu2020gcc,zhu2020deep,hassani2020contrastive,liu2023graphprompt,yu2024generalized}, $T=(\{\mathcal{A}_u:u\in V\},\{\mathcal{B}_u:u\in V\})$, where $\mathcal{A}_u$ is the set of positive instances for node $u$, and $\mathcal{B}_u$ is the set of negative instances for $u$. Its loss function $\mathcal{L}_T$ can be standardized \cite{yu2024generalized} to a similar form as follows.
\begin{align}\label{eq.contra-loss}
      \mathcal{L}_{T} &= -\sum_{u \in V} \ln P(u,\mathcal{A}_u,\mathcal{B}_u),\\
      P(u,\mathcal{A}_u,\mathcal{B}_u)&\triangleq\frac{\sum_{a \in \mathcal{A}_u} \mathtt{sim}(\vec{h}_u, \vec{h}_a)}{\sum_{a \in \mathcal{A}_u} \mathtt{sim}(\vec{h}_u, \vec{h}_a)+\sum_{b \in \mathcal{B}_u} \mathtt{sim}(\vec{h}_u, \vec{h}_b)},
\end{align}
where $\mathtt{sim}(\cdot,\cdot)$ represents a similarity function such as cosine similarity 
in our experiments. The optimization objective of task $T$ in Eq.~\eqref{eq.contra-loss} is to maximize the similarity between $u$ and its positive instances while minimizing the similarity between $u$ and its negative instances. 
Based on this loss, we further propose the definitions of \emph{homophily tasks} and \emph{homophily samples}.

\begin{definition}[Homophily Task] \label{def.homphily-loss}
On a graph $G=(V,E)$, a pre-training task $T=(\{\mathcal{A}_u:u\in V\},\{\mathcal{B}_u:u\in V\})$ is a homophily task if and only if, $\forall u\in V, \forall a \in \mathcal{A}_u, \forall b \in \mathcal{B}_u, (u, a) \in E \land (u, b) \notin E$. A task that is not a homophily task is called a non-homophily task.\qed
\end{definition}

In particular, the widely used link prediction task \cite{peng2020graph,liu2023graphprompt,yu2023hgprompt,yu2023multigprompt,nguyen2024diffusion,yu2024generalized} is a homophily task, where $\mathcal{A}_u$ is a subset of nodes linked to $u$ and $\mathcal{B}_u$ is a subset of nodes not linked to $u$. 

\begin{definition}[Homophily Sample] \label{def.homphily-sample}
On a graph $G=(V,E)$, consider a triplet $(u, a, b)$ where $u \in V$, $(u, a) \in E$ and $(u, b) \notin E$. The triplet $(u, a, b)$ is a homophily sample if and only if $\mathtt{sim}(\vec{h}_u, \vec{h}_a) > \mathtt{sim}(\vec{h}_u, \vec{h}_b)$, and it is a non-homophily sample otherwise. \qed
\end{definition}

Subsequently, we can establish the following theorems. 
\begin{theorem}\label{theorem.loss}
    For a homophily task $T$, 
    adding a homophily sample always results in a smaller loss than adding a non-homophily sample.
\end{theorem}
\begin{proof}
Consider a homophily sample $(u, a, b)$ for some $(u,a) \in E$ and $(u,b) \notin E$,
as well as a non-homophily sample $(u, a', b')$ for some $(u,a') \in E$ and $(u,b') \notin E$.
Let the overall loss with  $(u, a, b)$ be $L_T$, and that with $(u, a', b')$ be $L'_T$.
Since $(u,a,b)$ is a homophily sample, we have $\mathtt{sim}(\vec{h}_u, \vec{h}_a) > \mathtt{sim}(\vec{h}_u, \vec{h}_b)$, and thus 
$$p(u,a,b)=\textstyle\frac{ \mathtt{sim}(\vec{h}_u, \vec{h}_a)}{ \mathtt{sim}(\vec{h}_u, \vec{h}_a)+ \mathtt{sim}(\vec{h}_u, \vec{h}_b)}>0.5.$$
Moreover, since $(u,a',b')$ is a non-homophily sample, we have $\mathtt{sim}(\vec{h}_u, \vec{h}_a') \le \mathtt{sim}(\vec{h}_u, \vec{h}_b')$, and thus $p(u,a',b')\le 0.5$. Hence, $p(u,a,b)>p(u,a',b')$, implying that $L_T<L'_T$.
\end{proof}

\begin{theorem}\label{theorem.sample}
Consider a graph $G=(V,E)$ with a label mapping function $V \rightarrow Y$, where $y_v\in Y$ is the label mapped to $v\in V$. Suppose that the label mapping and node similarity are consistent, \ie, $$\forall u, a, b \in V, y_u = y_a \land y_u \neq y_b \Rightarrow \mathtt{sim}(\vec{h}_u, \vec{h}_a) > \mathtt{sim}(\vec{h}_u, \vec{h}_b).$$ Let $\mathbb{E}_T$ denote the expected number of homophily samples for a homophily task $T$ on the graph $G$.
Then, $\mathbb{E}_T$ increases monotonically as the homophily ratio $\mathcal{H}(G)$ defined w.r.t.~$Y$ increases.
\end{theorem}
\begin{proof}
For a homophily task $T=(\{\mathcal{A}_u:u\in V\},\{\mathcal{B}_u:u\in V\})$,
a triplet $(u,a,b)$ for some $u\in V, a\in \mathcal{A}_u$ and $b\in \mathcal{B}_u$ is a homophily sample with a probability of $P(y_u=y_a)(1-P(y_u=y_b))$, since $y_u=y_a\land y_u\neq y_b$ implies $\mathtt{sim}(\vec{h}_u, \vec{h}_a) > \mathtt{sim}(\vec{h}_u, \vec{h}_b)$. Hence, the expected number of homophily samples for $T$ is
\begin{align}
   \mathbb{E}_T =\sum_{u\in V} |\mathcal{A}_u||\mathcal{B}_u|P(y_u=y_a)(1-P(y_u=y_b)).
\end{align}
For a constant number of nodes with label $y_u$, as $\mathcal{H}(G)$ increases, $P(y_u=y_a)$ increases while $P(y_u=y_b)$ decreases, leading to a larger $\mathbb{E}_T$.
\end{proof}

In the next part, the theorems will guide us in choosing the appropriate pre-training tasks for non-homophilic graphs.

\subsection{Non-homophilic Graph Pre-training}
Consider a homophily task $T$. Following Theorem~\ref{theorem.sample}, non-homophilic graphs with lower homophily ratios are expected to have fewer homophily samples and more non-homophily samples for $T$. Furthermore, based on Theorem~\ref{theorem.loss}, adding a non-homophily sample results in a larger loss than adding a homophily sample. 
Consequently, for non-homophilic graphs, especially those with low homophily ratio, 
homophily tasks are not optimal for working with standard contrastive training losses, whereas non-homophily tasks may offer a better alternative.

We revisit mainstream graph pre-training methods and categorize them into two categories: homophily methods that employ homophily tasks, and non-homophily methods that do not.
Specifically, GPPT \cite{sun2022gppt}, GraphPrompt \cite{liu2023graphprompt} and HGPrompt \cite{yu2023hgprompt} 
are all homophily methods, since their pre-training tasks utilizes a form of link prediction, where $\mathcal{A}_u$ is a set of nodes linked to $u$, and $\mathcal{B}_u$ is a set of nodes not linked from $u$. In contrast, DGI \cite{velickovic2019deep}, GraphCL \cite{you2020graph}, 
and GraphACL \cite{xiao2024simple} are non-homophily methods, since $\mathcal{A}_u$ or $\mathcal{B}_u$ in their pre-training tasks is not related to the connectivity with $u$. Further details of these methods are shown in Appendix~\ref{sec.app.homo-task}.

In our experiments, we find that GraphCL \cite{you2020graph}, a non-homophily pre-training method, performs well across most graphs, including non-homophilic ones. We also experiment with link prediction 
\cite{liu2023graphprompt,sun2022gppt} and GraphACL \cite{xiao2024simple} to study their effects on different graphs. 

\section{Non-homophilic Prompt Learning}

\begin{figure*}[t]
\centering
\includegraphics[width=0.82\linewidth]{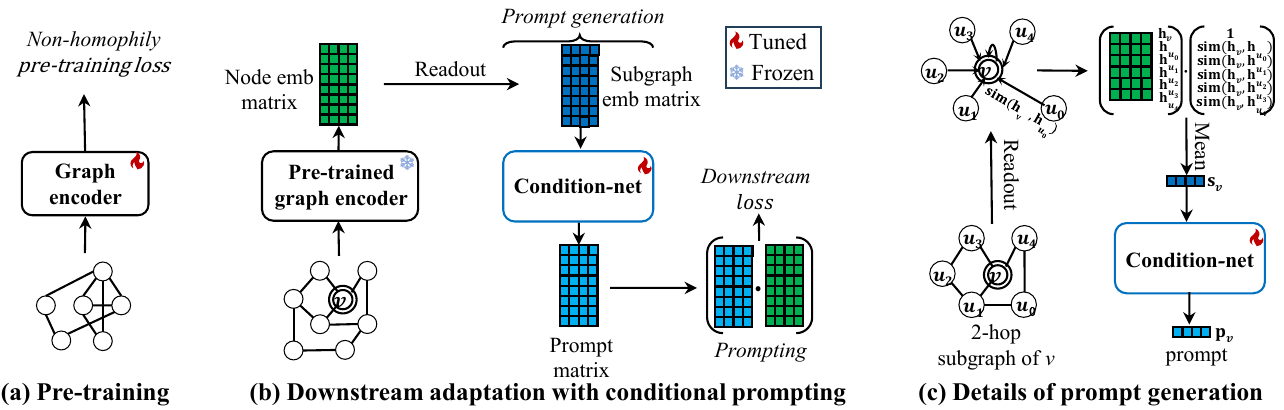}
\vspace{-2mm}
\caption{Overall framework of \model.}
\label{fig.framework}
\end{figure*}

In this section, we propose \model, a prompt learning framework for non-homophilic graphs. 
We first introduce the overall framework, and then develop the prompt generation and tuning process. Finally, we analyze the complexity of the proposed algorithm.

\subsection{Overall Framework}
We illustrate the overall framework of \model\ in Fig.~\ref{fig.framework}. It involves two stages: (a) graph pre-training and (b) downstream adaptation. In graph pre-training, we pre-train a graph encoder using a non-homophilic pre-training task, as shown in Fig.~\ref{fig.framework}(a). Subsequently, to adapt the pre-trained model to downstream tasks, we propose a conditional network (condition-net) that generates a series of prompts, as depicted in Fig.~\ref{fig.framework}(b). As a result, each node is equipped with its own prompt, which can be used to modify its features to align with the downstream task. 
More specifically, the prompt generation is conditioned on the unique patterns of each node, in order to achieve fine-grained adaptation catering to the diverse non-homophilic characteristics of each node, as detailed in Fig.~\ref{fig.framework}(c).

\subsection{Prompt Generation and Tuning}
\stitle{Prompt generation.} 
In non-homophilic graphs, different nodes are characterized by unique non-homophilic patterns. Specifically, different nodes typically have diverse homophily ratios $\mathcal{H}(v)$, indicating distinct topological structures linking to their neighboring node. Moreover, even nodes with similar homophily ratios may have different neighborhood distributions in terms of the varying homophily ratios of the neighboring nodes.   
Therefore, instead of learning a single prompt for all nodes as in standard graph prompt learning \cite{liu2023graphprompt,sun2022gppt,sun2023all,yu2024generalized}, we design a condition-net \cite{zhou2022conditional} to generate a series of non-homophilic pattern-conditioned prompts. Consequently, each node is equipped with its own unique prompt, aiming to adapt to its distinct non-homophilic characteristics.

First, the non-homophilic patterns of a node can be characterized by considering a multi-hop neighborhood around the node. Specifically, given a node $v$, we readout their $\delta$-hop ego-network $S_v$, which is an induced subgraph containing the node $v$ and nodes reachable from $v$ in at most $\delta$ steps. Inspired by GGCN \cite{yan2022two}, 
the readout is weighted by the similarity between $v$ and their neighbors, as shown in Fig.~\ref{fig.framework}(c)
, obtaining a representation of the subgraph $S_v$ given by 
\begin{align}\label{eq.readout}
    \vec{s}_v=\frac{1}{|V(S_v)|} \sum_{u\in V(S_v)}\vec{h}_u\cdot \mathtt{sim}(\vec{h}_u,\vec{h}_v),
\end{align}
where $V(S_v)$ denotes the set of nodes in $S_v$. In our experiment, we set $\delta=2$ to balance between efficiency and capturing more unique non-homophilic patterns in the neighborhood of $v$.

Next, for each downstream task, our goal is to assign a unique prompt vector to each node. However, directly parameterizing these prompt vectors would significantly increase the number of learnable parameters, which may overfit  to the lightweight supervision in few-shot settings.
To cater to the unique non-homophilic characteristics of each node with minimal parameters, we propose to employ a condition-net \cite{zhou2022conditional} to generate node-specific prompt vectors. Specifically, conditioned on the subgraph readout $\vec{s}_v$ of a node $v$, the condition-net generates a unique prompt vector for $v$ w.r.t. a task $t$, denoted by 
$\vec{p}_{t,v}$, as follows.
\begin{align}\label{eq.prompt-generation}
    \vec{p}_{t,v}=\mathtt{CondNet}(\vec{s}_v;\phi_t),
\end{align}
where $\mathtt{CondNet}$ is the condition-net parameterized by $\phi_t$. It outputs a unique prompt vector $\vec{p}_{t,v}$, which varies based on the input $s_v$ that characterizes the non-homophily patterns of node $v$.
Note that this is a form of hypernetworks \cite{ha2016hypernetworks}, which employs a secondary network to generate the parameters for the main network conditioned on the input feature. In our context, the condition-net is the secondary network, generating prompt parameters without expanding the number of learnable parameters in the main network. The secondary network $\mathtt{CondNet}$ can be any learnable function, such as a fully-connected layer or a multi-layer perceptron (MLP). We employ an MLP with a compact bottleneck architecture \cite{wu2018reducing}. 

Subsequently,  we perform fine-grained, node-wise adaptation to task $t$. Concretely, the prompt $\vec{p}_{t,v}$ for node $v$ is employed to adjust $v$'s features or its embeddings in the hidden or output layers  \cite{yu2024generalized}.
In our experiments, we choose a simple yet effective implementation that modifies the nodes' output embeddings through an element-wise product, as follows. 
\begin{align}\label{eq.prompt}
\tilde{\vec{h}}_{t,v}=\vec{p}_{t,v} \odot \vec{h}_v,
\end{align}
where the prompt $\vec{p}_{t,v}$ is generated with an equal dimension as $\vec{h}_v$.

\stitle{Prompt tuning.}\label{sec.model.tuning}
In this work, we focus on two common types of downstream task: node classification and graph classification. The prompt tuning process does not directly optimize the prompt vectors; instead it optimizes the condition-net, which subsequently generates the prompt vectors, for a given downstream task. 

We utilize a loss function based on node/graph similarity following previous work \cite{liu2023graphprompt,yu2024generalized}.
Formally, for a task $t$ with a labeled training set $\mathcal{D}_t=\{(x_1, y_1), (x_2, y_2), \ldots\}$, where $x_i$ can be either a node or a graph, and $y_i \in Y$ is $x_i$'s class label from a set of classes $Y$. The downstream loss function is
\begin{align}\label{eq.prompt-loss}
    \bL_{\text{down}}(\phi_t)=
    -\sum_{(x_i,y_i)\in \mathcal{D}_t}\ln\frac{\exp\left(\frac{1}{\tau}\text{sim}(\vec{\tilde{h}}_{t,x_i},\vec{\bar{h}}_{t,y_i})\right)}{\sum_{c\in Y}\exp\left(\frac{1}{\tau}\text{sim}(\vec{\tilde{h}}_{t,x_i},\vec{\bar{h}}_{t,c})\right)},
\end{align}
where $\vec{\tilde{h}}_{t,x_i}$ denotes the output embedding of node $v$/graph $G$ for task $t$. Specifically, for node classification $\vec{\tilde{h}}_{t,{v}}$ is the output embedding in Eq.~\eqref{eq.prompt}; for graph classification, $\vec{\tilde{h}}_{t,{G}}=\sum_{u\in V}\vec{\tilde{h}}_{t,{u}}$, involving an additional graph readout.
The prototype embedding for class \( c \), \( \vec{\bar{h}}_{t, c} \), is the average of the embedding of all labeled nodes/graphs belonging to class $c$.

During prompt tuning, we update only the lightweight parameters of the condition-net ($\phi_t$), while freezing the pre-trained GNN weights. Thus, our approach is parameter-efficient and amenable to few-shot settings, where \( \mathcal{D}_t \) contains only a small number of training examples for task $t$.

\subsection{Algorithm and Complexity Analysis}
\stitle{Algorithm.} We detail the main steps for the conditional prompt generation and tuning in Algorithm~\ref{alg.prompt}, Appendix~\ref{app.alg}. 

\stitle{Complexity analysis.}
For a downstream graph \( G \), the computational process of \model\ involves two main parts: encoding nodes via a pre-trained GNN, and conditional prompt learning.
The first part's complexity is determined by the GNN architecture, akin to other methods employing a pre-trained GNN. In a standard GNN, each node aggregates features from up to \( D \) neighbors per layer. Thus, the complexity of calculating node embeddings over \( L \) layers is \( O(D^L \cdot |V|) \), where \( |V| \) denotes the number of nodes.
The second part, conditional prompt learning, has two stages: prompt generation and prompt tuning. In prompt generation, each subgraph embedding is fed into the condition-net. The subgraph embedding of each node involves a readout from the $\delta$-hop neighborhood, resulting in a complexity of \( O(D^\delta \cdot |V|) \) with at most \( D \) neighbors per hop. During prompt tuning, each node in \( G \) is adjusted using a prompt vector, with a complexity of \( O(|V|) \). Therefore, the total complexity for conditional prompt learning is \( O(D^\delta \cdot |V|) \).

In conclusion, the overall complexity of \model\ is \( O((D^L+D^\delta) \cdot |V|) \). 
As both $L$ and $\delta$ are small constants, the two parts have comparable complexity. That is, the proposed conditional prompt learning does not increase the order of complexity relative to the pre-trained GNN encoder, if $\delta$ is chosen to be no larger than $L$.

\section{Experiments}
In this section, we conduct experiments to evaluate \model, and analyze the empirical results.


\begin{table*}[tbp]
    \centering
    \small
    \caption{Accuracy evaluation on one-shot node classification. 
    }
    \vspace{-2mm}
    \label{table.node-classification}%
    \resizebox{1\linewidth}{!}{%
    \begin{tabular}{@{}l|c|c|c|c|c|c|c|c@{}}
    \toprule
   Methods & Wisconsin & Squirrel & Chameleon & Cornell & PROTEINS & ENZYMES & Citeseer & Cora \\\midrule\midrule
    \method{GCN} 
    & 21.39 $\pm$ \phantom{0}6.56 & 20.00 $\pm$ \phantom{0}0.29
    & 25.11 $\pm$ \phantom{0}4.19 & 21.81 $\pm$ \phantom{0}4.71
    & 43.32 $\pm$ \phantom{0}9.35& 48.08 $\pm$ \phantom{0}4.71& 31.27 $\pm$ \phantom{0}4.53 & 28.57 $\pm$ \phantom{0}5.07\\ 
    \method{GAT} 
    & 28.01 $\pm$ \phantom{0}5.40 & 21.55 $\pm$ \phantom{0}2.30
    & 24.82 $\pm$ \phantom{0}4.35 & 23.03 $\pm$ 13.19
    &  31.79 $\pm$ 20.11&  35.32 $\pm$ 18.72&  30.76 $\pm$ \phantom{0}5.40&  28.40 $\pm$ \phantom{0}6.25\\
    \method{H2GCN} 
    & 23.60 $\pm$ \phantom{0}4.64 & 21.90 $\pm$ \phantom{0}2.15
    & 25.89 $\pm$ \phantom{0}4.96 & 32.77 $\pm$ 14.88
    & 29.60 $\pm$ \phantom{0}6.99 & 37.27 $\pm$ \phantom{0}8.73
    & 26.98 $\pm$ \phantom{0}6.25 & 34.58 $\pm$ \phantom{0}9.43
\\
    \method{FAGCN} 
    & 35.03 $\pm$ 17.92 & 20.91 $\pm$ \phantom{0}1.79
    & 22.71 $\pm$ \phantom{0}3.74 & 28.67 $\pm$ 17.64
    & 32.63 $\pm$ \phantom{0}9.94 & 35.87 $\pm$ 13.47
    & 26.46 $\pm$ \phantom{0}6.34 & 28.28 $\pm$ \phantom{0}9.57
\\\midrule
    \method{DGI}
    & 28.04 $\pm$ \phantom{0}6.47 & 20.00 $\pm$ \phantom{0}1.86
    & 19.33 $\pm$ \phantom{0}4.57 & 32.54 $\pm$ 15.66
    & 45.22 $\pm$ 11.09& 48.05 $\pm$ 14.83
    & 45.00 $\pm$ \phantom{0}9.19& 54.11 $\pm$ \phantom{0}9.60\\
    \method{GraphCL}
    & 29.85 $\pm$ \phantom{0}8.46 & 21.42 $\pm$ \phantom{0}2.22
    & 27.16 $\pm$ \phantom{0}4.31 & 24.69 $\pm$ 14.06
    & 46.15 $\pm$ 10.94 & 48.88 $\pm$ 15.98 
    & 43.12 $\pm$ \phantom{0}9.61& 51.96 $\pm$ \phantom{0}9.43
\\
    \method{DSSL}
    & 28.46 $\pm$ 10.31 & 20.94 $\pm$ \phantom{0}1.88
    & \underline{27.92} $\pm$ \phantom{0}3.93& 20.36 $\pm$ \phantom{0}5.38
    & 40.42 $\pm$ 10.08 & \underline{66.59} $\pm$ 19.28
    & 39.86 $\pm$ \phantom{0}8.60 & 40.79 $\pm$ \phantom{0}7.31
\\
    \method{GraphACL}
    & \underline{34.57} $\pm$ 10.46 & \underline{24.44} $\pm$ \phantom{0}3.94& 26.72 $\pm$ \phantom{0}4.67 & \underline{33.17} $\pm$ 16.06& 42.16 $\pm$ 13.50 & 47.57 $\pm$ 14.36
    & 35.91 $\pm$ \phantom{0}7.87 & 46.65 $\pm$ \phantom{0}9.54
\\\midrule 
    \method{GPPT}
    & 27.39 $\pm$ \phantom{0}6.67 & 20.09 $\pm$ \phantom{0}0.91
    & 24.53 $\pm$ \phantom{0}2.55 & 25.09 $\pm$ \phantom{0}2.92
    &  35.15 $\pm$ 11.40 &  35.37 $\pm$ \phantom{0}9.37&  21.45 $\pm$ \phantom{0}3.45 &  15.37 $\pm$ \phantom{0}4.51\\
    \method{GraphPrompt}
    & 31.48 $\pm$ \phantom{0}5.18 & 21.22 $\pm$ \phantom{0}1.80
    & 25.36 $\pm$ \phantom{0}3.99 & 31.00 $\pm$ 13.88
    & \underline{47.22} $\pm$ 11.05& 53.54 $\pm$ 15.46& \underline{45.34} $\pm$ 10.53& \underline{54.25} $\pm$ \phantom{0}9.38\\
    \method{GraphPrompt+}
    & 31.54 $\pm$ \phantom{0}4.54 & 21.24 $\pm$ \phantom{0}1.82
    & 25.73 $\pm$ \phantom{0}4.50 & 31.65 $\pm$ 14.48
    & 46.08 $\pm$ \phantom{0}9.96 & 57.68 $\pm$ 13.12
    & 45.23 $\pm$ 10.01 & 52.51 $\pm$ \phantom{0}9.73
\\\midrule
    \method{\model}
    & \textbf{44.72} $\pm$ 11.93& \textbf{24.59} $\pm$ \phantom{0}3.41& \textbf{30.67} $\pm$ \phantom{0}3.73& \textbf{37.90} $\pm$ \phantom{0}9.31& \textbf{48.95} $\pm$ 10.85& \textbf{72.94} $\pm$ 20.23& \textbf{49.02} $\pm$ 10.66& \textbf{57.92} $\pm$ 11.50\\    \bottomrule
        \end{tabular}}
        \\
   \parbox{1\textwidth}{\footnotesize Results are reported in percent. The best method is bolded and the runner-up is underlined.}
\end{table*}

\begin{table*}[tbp]
    \centering
    \small
    \caption{Accuracy evaluation on one-shot graph classification. 
    }
    \vspace{-2mm}
    \label{table.graph-classification}%
    \resizebox{1\linewidth}{!}{%
    \begin{tabular}{@{}l|c|c|c|c|c|c|c|c@{}}
    \toprule
   Methods & Wisconsin & Squirrel & Chameleon & Cornell & PROTEINS & ENZYMES & BZR & COX2 \\\midrule\midrule
    \method{GCN} 
    & 21.39 $\pm$ \phantom{0}6.56 & 11.77 $\pm$ \phantom{0}3.10
    & 17.21 $\pm$ \phantom{0}4.80 & 26.36 $\pm$ \phantom{0}4.35
    & 51.66 $\pm$ 10.87 & 19.30 $\pm$ \phantom{0}6.36
    & 45.06 $\pm$ 16.30 & 43.84 $\pm$ 13.94
\\ 
    \method{GAT} 
    & 24.93 $\pm$ \phantom{0}7.59 & 20.70 $\pm$ \phantom{0}1.51
    & 25.71 $\pm$ \phantom{0}3.32 & 22.66 $\pm$ 12.46
    & 51.33 $\pm$ 11.02 & 20.24 $\pm$ \phantom{0}6.39
    & 46.28 $\pm$ 15.26 & 51.72 $\pm$ 13.70
\\
    \method{H2GCN}
    & 22.23 $\pm$ \phantom{0}6.38 & 20.69 $\pm$ \phantom{0}1.42
    & \underline{26.76} $\pm$ \phantom{0}3.98 & 23.11 $\pm$ 11.78
    & 53.81 $\pm$ \phantom{0}8.85 & 19.40 $\pm$ \phantom{0}5.57
    & 50.28 $\pm$ 12.13 & 53.70 $\pm$ 11.73
\\
    \method{FAGCN}
    & 23.81 $\pm$ \phantom{0}9.50 & 20.83 $\pm$ \phantom{0}1.43
    & 25.93 $\pm$ \phantom{0}4.03 & 25.71 $\pm$ 13.12
    & 55.45 $\pm$ 11.57 & 19.95 $\pm$ \phantom{0}5.94
    & {50.93} $\pm$ 12.41 & 50.22 $\pm$ 11.50
\\\midrule
    \method{DGI}
    & \underline{29.77} $\pm$ \phantom{0}6.22 & 20.50 $\pm$ \phantom{0}1.52
    & 24.29 $\pm$ \phantom{0}4.33 & 18.60 $\pm$ 12.79
    & 50.32 $\pm$ 13.47 & 21.57 $\pm$ \phantom{0}5.37
    & 49.97 $\pm$ 12.63 & 54.84 $\pm$ 14.76
\\
    \method{GraphCL}
    & 27.93 $\pm$ \phantom{0}5.27 & \underline{21.01} $\pm$ \phantom{0}1.86
    & 26.45 $\pm$ \phantom{0}4.30 & 20.03 $\pm$ 10.05
    & 54.81 $\pm$ 11.44 & 19.93 $\pm$ \phantom{0}5.65
    & 50.50 $\pm$ 18.62 & 47.64 $\pm$ 22.42
\\
    \method{DSSL}
    & 22.05 $\pm$ \phantom{0}3.90 & 20.74 $\pm$ \phantom{0}1.61
    & 26.19 $\pm$ \phantom{0}3.72 & 18.38 $\pm$ 10.63
    & 52.73 $\pm$ 10.98 & \textbf{23.14} $\pm$ \phantom{0}6.71
    & 49.04 $\pm$ \phantom{0}8.75 & 54.23 $\pm$  14.17
\\
    \method{GraphACL}
    & 22.98 $\pm$ \phantom{0}5.89 & 20.80 $\pm$ \phantom{0}1.28
    & 26.28 $\pm$ \phantom{0}3.93 & \underline{26.50} $\pm$ 17.18
    & \textbf{56.11} $\pm$ 13.95 & 20.28 $\pm$ \phantom{0}5.60
    & 49.24 $\pm$ 17.87 & 49.59 $\pm$ 23.93
\\\midrule 
    \method{GraphPrompt}
    & 28.34 $\pm$ \phantom{0}3.89 & \textbf{21.22} $\pm$ \phantom{0}1.80
    & 26.51 $\pm$ \phantom{0}4.67 &  24.06 $\pm$ 13.71 
    & 53.61 $\pm$ \phantom{0}8.90 & 21.85 $\pm$ \phantom{0}6.17
    & 50.46 $\pm$ 11.46  & \underline{55.01} $\pm$ 15.23
\\
    \method{GraphPrompt+} 
    & 26.95 $\pm$ \phantom{0}7.42 & 20.80 $\pm$ \phantom{0}1.45 & 26.03 $\pm$ \phantom{0}4.17 &
    25.31 $\pm$ \phantom{0}7.65& 54.55 $\pm$ 12.61 & 21.85 $\pm$ \phantom{0}5.15
    & \textbf{53.26} $\pm$ 14.99 & 54.73 $\pm$ 14.58
\\\midrule
    \method{\model}
    & \textbf{31.54} $\pm$ \phantom{0}5.30 & 20.92 $\pm$ \phantom{0}1.37
    & \textbf{28.50} $\pm$ \phantom{0}5.30 & \textbf{27.17} $\pm$ \phantom{0}9.58
    & \textbf{56.11} $\pm$ 10.19 & \underline{22.55} $\pm$ \phantom{0}6.70
    & \underline{51.62} $\pm$ 14.27 & \textbf{56.46} $\pm$ 14.57
\\    \bottomrule
        \end{tabular}}
        \\
\end{table*}

\subsection{Experimental Setup}
\stitle{Datasets.}
We conduct experiments on ten benchmark datasets.
\textit{Wisconsin} \cite{pei2020geom}, \textit{Cornell} \cite{pei2020geom}, \textit{Chameleon} \cite{rozemberczki2021multi}, and \textit{Squirrel} \cite{rozemberczki2021multi} are all web graphs. Each dataset features a single graph, where nodes correspond to web pages and edges represent hyperlinks connecting these pages.
\textit{Cora} \cite{mccallum2000automating} and \textit{Citeseer} \cite{sen2008collective} are citation networks. They consist of a single graph each, with nodes representing academic papers and edges indicating citation relationships.
\textit{PROTEINS} \cite{borgwardt2005protein} consists of a series of protein graphs. Nodes in these graphs denote secondary structures, while edges depict neighboring relationships either within the amino acid sequence or in three-dimensional space.
\textit{ENZYMES} \cite{wang2022faith}, \textit{BZR} \cite{nr}, and \textit{COX2} \cite{nr} are collections of molecular graphs. These datasets describe enzyme structures from the BRENDA enzyme database, ligands related to benzodiazepine receptors, and cyclooxygenase-2 inhibitors, respectively.
We summarize these datasets in Table~\ref{table.datasets}, Appendix~\ref{app.dataset}.

\stitle{Baselines.}
We evaluate \model\ against a series of state-of-the-art methods from the following three categories.
(1) \emph{End-to-end GNNs}: GCN \cite{kipf2016semi}, GAT \cite{velivckovic2017graph}, H2GCN \cite{zhu2020beyond}, and FAGCN \cite{bo2021beyond} are trained in a supervised manner directly using downstream labels. Specifically, GCN and GAT are designed for homophilic graphs, whereas H2GCN is developed for heterophilic graphs, and FAGCN for non-homophilic graphs.
(2) \emph{Graph pre-training models}: DGI \cite{velickovic2019deep}, GraphCL \cite{you2020graph}, DSSL \cite{xiao2022decoupled}, and GraphACL \cite{xiao2024simple} follow the ``pre-train, fine-tune'' paradigm. 
(3) \textit{Graph prompt learning models}:  GPPT \cite{sun2022gppt}, GraphPrompt \cite{liu2023graphprompt}, and GraphPrompt+ \cite{yu2024generalized} employ self-supervised pre-training tasks, and use the same prompts for all nodes in downstream adaptation. 
Note that GPPT is specifically designed for node classification and cannot be directly used for graph classification. Therefore, we evaluate GPPT on node classification tasks only.

Note that some graph few-shot learning methods, such as Meta-GNN \cite{zhou2019meta}, AMM-GNN \cite{wang2020graph}, RALE \cite{liu2021relative}, VNT \cite{tan2023virtual}, and ProG \cite{sun2023all}, are based on the meta-learning paradigm \cite{finn2017model}. They require a set of labeled base classes in addition to the few-shot classes, and thus are not compared here.

\stitle{Parameter settings.}
For all baselines, we use the original authors' code and reference their recommended settings, while further tuning their hyperparameters to ensure optimal performance. Detailed descriptions of the implementations and settings for both the baselines and our \model\ are provided in Appendix~\ref{app.dataset}.

\stitle{Downstream tasks and evaluation.}
We conduct two types of downstream task: \textit{node classification}, and \textit{graph classification}. These tasks are set up as $k$-shot classification problems, meaning that for each class, $k$ instances (nodes or graphs) are randomly selected for supervision.
The low-homophily datasets, \ie, \textit{Wisconsin}, \textit{Squirrel}, \textit{Chameleon} and \textit{Cornell}, only comprise a single graph and cannot be directly used for graph classification. Thus, following previous research \cite{lu2021learning,yu2023hgprompt}, we generate multiple graphs by constructing ego-networks centered on the labeled nodes in each dataset. 
We then perform graph classification on these ego-networks, each labeled according to its ego node. Among datasets with high homophily ratios, \textit{PROTEINS}, \textit{ENZYMES}, \textit{BZR} and \textit{COX2} have ground-truth graph labels, which we employ directly for graph classification. 

Since the $k$-shot tasks are balanced classification problems, we use accuracy to evaluate the performance, in line with prior studies \cite{liu2023graphprompt, wang2020graph, liu2021relative, yu2024generalized}. We pre-train the graph encoder once for each dataset and then use the same pre-trained model for all downstream tasks.
We generate 100 $k$-shot tasks for both node classification and graph classification by repeating the sampling process 100 times. Each task is executed with five different random seeds, leading to a total of 500 results per task type. We report the mean and standard deviation of these 500 outcomes.

\begin{table*}[tb]
    \centering
    \small
    \addtolength{\tabcolsep}{-1mm}
    \caption{Ablation study on the effects of key components.}
    \vspace{-2mm}
    \label{table.ablation}%
    \resizebox{1\linewidth}{!}{%
    \begin{tabular}{@{}l|cccccc|cccccc@{}}
    \toprule
    \multirow{2}*{Methods} &\multicolumn{6}{c|}{Node classification} &\multicolumn{6}{c}{Graph classification}\\
    & Wisconsin & Squirrel & Chameleon & PROTEINS & ENZYMES & Citeseer  
    & Wisconsin & Squirrel & Chameleon & PROTEINS & ENZYMES & COX2\\
    \midrule\midrule
    \method{NoPrompt} 
    &25.41\text{\scriptsize ±\ 3.13}  &20.60\text{\scriptsize ±1.30} &22.71\text{\scriptsize ±3.54 } 
    &47.22\text{\scriptsize ±11.05} &66.59\text{\scriptsize ±19.28 } &43.12\text{\scriptsize ±\ 9.61}  
    &20.85\text{\scriptsize ±6.74} &20.18\text{\scriptsize ±1.30} &22.34\text{\scriptsize ±4.15} 
    &53.61\text{\scriptsize ±\ 8.90} &21.85\text{\scriptsize ±6.17}  &54.29\text{\scriptsize ±17.31}
\\  
    \method{SinglePrompt} 
    &32.76\text{\scriptsize ±\ 5.21} &20.85\text{\scriptsize ±1.32 } &22.78\text{\scriptsize ±3.35}  
    &30.33\text{\scriptsize ±19.59} &65.32\text{\scriptsize ±21.67} &48.64\text{\scriptsize ±10.09}  
    &25.77\text{\scriptsize ±6.24} &20.68\text{\scriptsize ±0.91} &27.03\text{\scriptsize ±3.98} 
    &56.35\text{\scriptsize ±10.59} &19.38\text{\scriptsize ±7.12}  &47.24\text{\scriptsize ±15.53}
\\  
    \method{NodeCond} 
    &35.56\text{\scriptsize ±\ 4.65}
    &21.26\text{\scriptsize ±3.95}
    &21.13\text{\scriptsize ±2.23}
    &36.01\text{\scriptsize ±19.70}&68.54\text{\scriptsize ±19.31}
    &48.30\text{\scriptsize ±10.22}
    &25.30\text{\scriptsize ±4.62}
    &20.98\text{\scriptsize ±1.56}
    &27.24\text{\scriptsize ±5.24}
    &\textbf{56.61}\text{\scriptsize ±10.03}
    &20.70\text{\scriptsize ±6.67}
    &55.92\text{\scriptsize ±14.66}
\\
    \method{\model$\textbackslash$sim} 
    &30.65\text{\scriptsize ±\ 4.05}
    &20.05\text{\scriptsize ±0.59}
    &20.96\text{\scriptsize ±4.21}
    &33.73\text{\scriptsize ±17.82}&36.02\text{\scriptsize ±20.64}
    &48.74\text{\scriptsize ±\ 2.66}
    &22.05\text{\scriptsize ±5.86}
    &19.93\text{\scriptsize ±0.42}
    &20.20\text{\scriptsize ±1.11}
    &52.30\text{\scriptsize ±10.94}
    &16.70\text{\scriptsize ±1.28}
    &50.05\text{\scriptsize ±17.67}
\\
    \method{\model} 
    & \textbf{44.72}\text{\scriptsize ±11.93} &\textbf{24.59}\text{\scriptsize ±3.41} &\textbf{30.67}\text{\scriptsize ±3.73}  &\textbf{48.95}\text{\scriptsize ±10.85} &\textbf{72.94}\text{\scriptsize ±20.23} &\textbf{49.02}\text{\scriptsize ±10.66}  
    &\textbf{31.54}\text{\scriptsize ±5.30} &\textbf{20.92}\text{\scriptsize ±1.37} &\textbf{28.50}\text{\scriptsize ±5.30}  
    &56.11\text{\scriptsize ±10.19} &\textbf{22.55}\text{\scriptsize ±6.70} &\textbf{56.46}\text{\scriptsize ±14.57} 	 
\\  
    \bottomrule
    \end{tabular}}
\end{table*}

\begin{table*}[tb]
    \centering
    \small
    \caption{Comparison between homophily and non-homophily tasks in pre-training.}
    \vspace{-2mm}
    \label{table.pre-train}%
    \begin{tabular}{@{}l|cccc|cccc@{}}
    \toprule
    \multirow{3}*{Pre-training task} &\multicolumn{4}{c|}{Node classification} &\multicolumn{4}{c}{Graph classification}\\
    & Wisconsin & Cornell  & PROTEINS & ENZYMES  
    & Wisconsin & Cornell  & PROTEINS & ENZYMES \\
    & 0.21 & 0.30  & 0.66 & 0.67  & 0.21 & 0.30  & 0.66 & 0.67 \\
    \midrule\midrule
    \method{Link Prediction \cite{sun2022gppt}}&23.01\text{±11.40 }   &26.27\text{±\phantom{0}7.61 } 
   &35.88\text{±\phantom{0}5.41} 
    &36.74\text{±\phantom{0}2.61}  &20.96\text{±\phantom{0}4.21} &25.38\text{±\phantom{0}2.50} 
    &51.50\text{±\phantom{0}6.02}  &17.47\text{±\phantom{0}4.04}
\\  
    \method{Link Prediction \cite{liu2023graphprompt}}
    &28.93\text{±11.74 }  &16.29\text{±\phantom{0}7.93 }  
    &\textbf{48.95}\text{±10.85}&\textbf{52.87}\text{±14.73}&23.15\text{±\phantom{0}5.67} &22.05\text{±13.80}
    &\textbf{55.83}\text{±10.87}&\textbf{22.23}\text{±\phantom{0}5.51}\\  
    \method{GraphACL \cite{xiao2024simple} }
     &33.91\text{±\phantom{0}9.04}
     &29.55\text{±12.30} 
       &44.08\text{±10.03}
     &50.57\text{±13.11} 
    &26.42\text{±\phantom{0}7.25} &26.15\text{±\phantom{0}3.87}
    &54.15\text{±10.58}
    &21.64\text{±\phantom{0}5.88}
\\
    \method{GraphCL \cite{you2020graph}}
    &\textbf{44.72}\text{±11.93}&\textbf{37.90}\text{±\phantom{0}9.31}&48.28\text{±11.09} &51.46\text{±13.93} 
    &\textbf{31.54}\text{±\phantom{0}1.37}&\textbf{27.17}\text{±\phantom{0}5.30}&53.91\text{±\phantom{0}5.51} &21.78\text{±12.12}
\\  
    \bottomrule
    \end{tabular}
\end{table*}

\subsection{Performance Evaluation}\label{sec.exp.per}
We first evaluate one-shot classification tasks. Then, we vary the number of shots to investigate their impact on performance.

\stitle{One-shot performance.}\label{exp.main}
We present the results of one-shot node and graph classification tasks on non-homophilic graphs in Tables~\ref{table.node-classification} and~\ref{table.graph-classification}, respectively. We make the following observations:
(1) \model\ surpasses the vast majority of baseline methods, outperforming the best competitor by up to 21.49\% on node classification and 6.50\% on graph classification. These results demonstrate its effectiveness in learning prior knowledge from non-homophilic graphs and capturing node-specific patterns.
(2) Other graph prompt learning methods, \ie, GPPT, GraphPrompt, and GraphPrompt+, significantly lag behind \model. Their suboptimal performance can be attributed to their inability to account for a variety of node-specific patterns. These results underscore the importance of our conditional prompting in characterizing node embeddings to capture the unique pattern of each node.
(3) GPPT is at best comparable to, and often performs worse than other baselines because it is not well suited to few-shot learning.

\stitle{Few-shot performance.}
To assess the performance of \model\ with different amounts of labeled data, we vary the number of shots in the node classification tasks. We present the results in Fig.~\ref{fig.fewshot-nc} with several competitive baselines for selected datasets. Note that given the limited number of nodes in \textit{Wisconsin}, we only conduct tasks up to $3$ shots.
We observe that
 \model\ generally outperforms these baselines in low-shot scenarios (\( k \leq 5 \)) by a significant margin, showcasing the effectiveness of our approach with limited labeled data.
Furthermore, as the number of shots increases, while all methods generally show improved performance, \model\ remains competitive and often surpasses the other methods. 
We focus on the one-shot setting for the remaining experimental results.

\begin{figure}[t]
\centering
\includegraphics[width=1\linewidth]{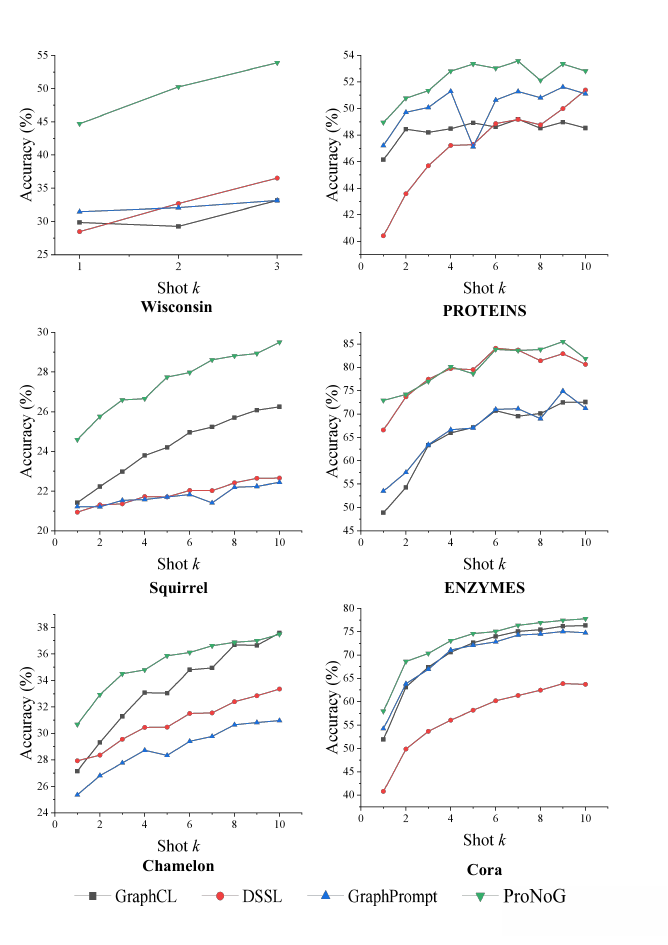}
\vspace{-5mm}
\caption{Impacts of different shots on node classification.}
\label{fig.fewshot-nc}
\end{figure}

\begin{figure}[t]
\centering
\includegraphics[width=1\linewidth]{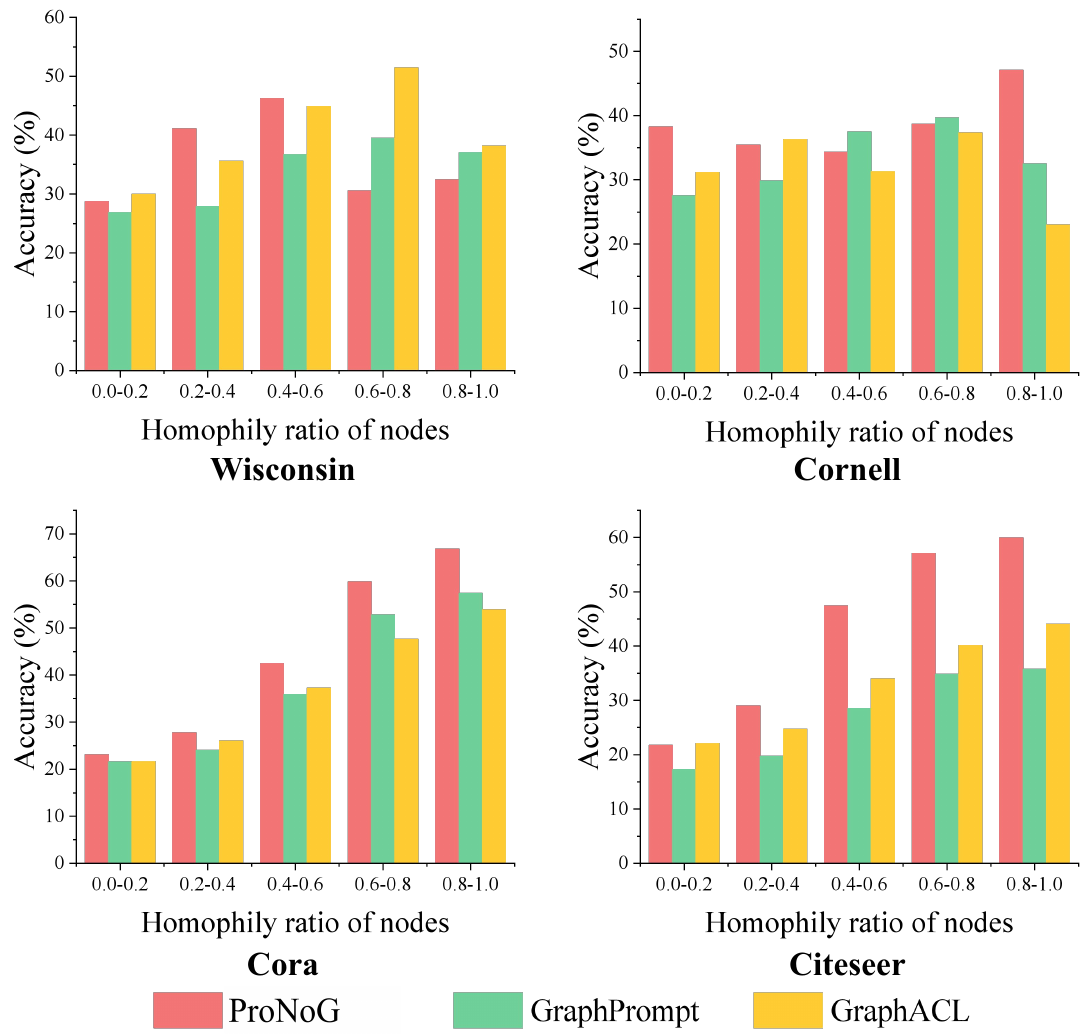}
\vspace{-5mm}
\caption{Results across nodes with varying homophily ratios.}
\label{fig.node-homophily-ratio}
\end{figure}

\subsection{Ablation Study}
To comprehensively understand the influence of conditional prompt learning in \model, we perform an ablation study comparing \model\ with four of its variants: (1) \method{NoPrompt} replaces conditional prompt learning with a classifier for downstream tasks; (2) \method{SinglePrompt} uses a single prompt instead of conditional prompts to modify all nodes; (3) \method{NodeCond} directly uses the output embedding of a node from the pre-trained graph encoder as input to the condition-net to generate the prompt, without reading out the subgraph in Eq.~\eqref{eq.readout}; (4) \method{\model$\textbackslash$sim} reads out the subgraph via a mean pooling without similarity weighting between the ego nodes and their neighbors as in Eq.~\eqref{eq.readout}.

As shown in Table~\ref{table.ablation}, \model\ consistently outperforms these variants in all but one instance, in which its performance is still competitive. This highlights the necessity of reading out subgraphs with similarity weighting in order to capture the characteristics of each node, and the advantage of using conditional prompt learning to adapt to each node.

\subsection{Analysis on Pre-Training Methods}
To investigate the effect of homophily and non-homophily tasks on pre-training, we employ two forms of link prediction from GPPT \cite{sun2022gppt} and GraphPrompt \cite{liu2023graphprompt} as the homophily tasks, as well as GraphCL \cite{you2020graph} and GraphACL \cite{xiao2024simple} as non-homophily tasks. 
To isolate their effects, we apply the same conditional prompt learning from \model\ for downstream adaptation.

We present the comparison in Table~\ref{table.pre-train}. It can be observed that, for graphs with lower homophily ratios (\ie, \textit{Wisconsin} and \textit{Cornell}), non-homophily tasks significantly outperform the homophily tasks. Conversely, for graphs with higher homophily ratios (\ie, \textit{PROTEINS} and \textit{ENZYMES}), the performance of homophily and non-homophily tasks becomes more comparable. While the homophily tasks with link prediction may have a slight advantage on highly homophilic graphs, non-homophily tasks are competitive across both homophilic and non-homophilic graphs. Hence, non-homophily tasks provide a more robust solution overall. 

\subsection{Analysis on Diverse Node Patterns}
To evaluate the ability of \model\ in capturing node-specific patterns, we investigate the classification accuracy of different node groups with varying homophily ratios, \ie, $[0.0,0.2)$, $[0.2,0.4)$, $[0.4,0.6)$, $[0.6,0.8)$, $[0.8,1.0]$. In each node group, we compare the performance of \model\ with several competitive baselines. 

As shown in Fig.~\ref{fig.node-homophily-ratio}, \model\ outperforms the baselines across various node groups, regardless of their unique characteristics reflected in different homophily ratios. These consistent improvements in all groups further demonstrate the effectiveness of \model\ in capturing diverse node patterns and highlight the advantage of our proposed conditional prompt learning.



\section{Conclusions}
In this paper, we explored pre-training and prompt learning on non-homophilic graphs. The goals are twofold: learning comprehensive knowledge irrespective of the varying non-homophilic patterns of graphs, and adapting the nodes with diverse distributions of non-homophily patterns to downstream tasks in a fine-grained, node-wise manner. We first revisit graph pre-training on non-homophilic graphs, providing theoretical insights into the choice of pre-training tasks. Then, for downstream adaptation, we proposed a condition-net to generate a series of prompts conditioned on node-specific non-homophilic patterns. Finally, we conducted extensive experiments on ten public datasets, showing that \model\ significantly outperforms diverse state-of-the-art baselines.


\section*{Acknowledgments}
This research / project is supported by the Ministry of Education, Singapore, under its Academic Research Fund Tier 2 (Proposal ID: T2EP20122-0041). Any opinions, findings and conclusions or recommendations expressed in this material are those of the author(s) and do not reflect the views of the Ministry of Education, Singapore. 

\bibliographystyle{ACM-Reference-Format}
\bibliography{references}

\appendix
\section*{Appendices}
\renewcommand\thesubsection{\Alph{subsection}}
\renewcommand\thesubsubsection{\thesubsection.\arabic{subsection}}
\subsection{Algorithm}\label{app.alg}
We detail the main steps for conditional prompt generation and tuning in Algorithm~\ref{alg.prompt}. In brief, we iterate through each downstream task to learn the corresponding prompt vectors. In lines 3--5, we compute the embedding for each node using the pre-trained graph encoder, with the pre-trained weights \( \Theta_0 \) frozen throughout the adaptation process.
In lines 6--22, we optimize the condition-net. Specifically, we perform similarity-weighted readout (lines 9--11), generate prompts (lines 12--13), modify nodes' embeddings using these prompts (lines 14--15), calculate the embedding of the corresponding graph (line 16), and update the embeddings for the prototypical nodes/graphs based on the few-shot labeled data provided in the task (lines 17--19).

\subsection{Homophily and Non-Homophily Methods}\label{sec.app.homo-task}
We provide further details about the set of positive samples $\mathcal{A}$ and negative samples $\mathcal{B}$ for various homophily and non-homophily methods in Table~\ref{table.pre-train-method}.

\begin{table*}[tbp]
\centering
\caption{Positive and negative samples for homophily and non-homophily contrastive methods. } \label{table.pre-train-method}
\vspace{-2mm}
\small
\begin{tabular}{c|c|c|c}
\toprule
 Pre-training task  & {Positive instances $\mathcal{A}_u$} &{Negative instances $\mathcal{B}_u$} & Homophily task\\
\midrule
     Link prediction \cite{liu2023graphprompt,yu2023hgprompt,yu2024generalized}
     & a node connected to node $u$
     & nodes disconnected to node $u$
     & Yes
     \\ 
     DGI \cite{velickovic2019deep}
     & nodes in graph $G$
     & nodes in corrupted graph $G'$
     & No
     \\ 
     {GraphCL \cite{you2020graph}} 
     & an augmented graph from graph $G$
     & augmented graphs from $G'\ne G$
     & No
     \\
     GraphACL \cite{xiao2024simple}
     & nodes with similar ego-subgraph to node $u$
     & nodes with dissimilar ego-subgraph to node $u$
     & No
     \\
 \bottomrule
\end{tabular}
\end{table*}

\begin{table}[tbp]
\center
\small
\addtolength{\tabcolsep}{-1mm}
\caption{Summary of datasets. 
\label{table.datasets}}
\vspace{-2mm}
\resizebox{1\columnwidth}{!}{%
\begin{tabular}{@{}c|rrrrrrr@{}}
\toprule
    & \makecell[c]{Graphs} &\makecell[c]{Homophily \\ ratio} & \makecell[c]{Graph \\ classes} & \makecell[c]{Avg.\\ nodes} & \makecell[c]{Avg. \\ edges} & \makecell[c]{Node \\ features} & \makecell[c]{Node \\ classes} \\
\midrule
     Wisconsin & 1 & 0.21 & - & 251 & 199 & 1,703 & 5 \\
     Squirrel & 1 & 0.22 & - & 5,201 & 217,073 & 2,089 & 5 \\
     Chameleon & 1 & 0.23 &- & 2,277 & 36,101 & 2,325 & 5 \\
     Cornell & 1 & 0.30 & - & 183 & 295 & 1,703 & 5 \\
     PROTEINS & 1,113 & 0.66 & 2 & 39.06 & 72.82 & 1 & 3 \\
     ENZYMES & 600 & 0.67 & 6 & 32.63 & 62.14 & 18 & 3 \\
     Citeseer & 1 & 0.74 & - & 3,327 & 4,732 & 3,703 & 6 \\ 
     Cora & 1 & 0.81 & - & 2,708 & 5,429 & 1,433 & 7 \\ 
     BZR & 405 & - & 2 & 35.75 & 38.36 & 3 & - \\
     COX2 & 467 & - & 2 & 41.22 & 43.45 & 3 & - \\
 \bottomrule
\end{tabular}}
\parbox{1\columnwidth}{\raggedright \footnotesize Homophily ratios are calculated by Eq.~\eqref{eq.graph-level-edge-homophily-ratio}. Note that \textit{BZR} and \textit{COX2} do not have any node label, and thus no homophily ratios can be calculated.
}
\end{table}

\subsection{Further Experimental Details} \label{app.dataset}
\stitle{Datasets.} We summarize the statistics of the ten datasets used in our experiments in Table~\ref{table.datasets}.

\begin{algorithm}[tbp]
\small
\caption{\textsc{Conditional Prompt Learning for \model}}
\label{alg.prompt}
\begin{algorithmic}[1]
    \Require Pre-trained graph encoder with parameters $\Theta_0$, a set of downstream tasks $\mathcal{T}=\{t_1,\ldots,t_n\}$.
    \Ensure Optimized parameters $\{\phi_{t_1},\ldots,\phi_{t_n}\}$ of $n$ condition-nets
    \For{$i\leftarrow 1$ to $n$}
        \State \slash* Encoding graphs via pre-trained graph encoder *\slash
        \For{each graph $G=(V, E,\vec{X})$ in task $t_i$}
            \State $\vec{H}\leftarrow \textsc{GraphEncoder}(G;\Theta_0)$
            \State $\vec{h}_v \leftarrow \vec{H}[v]$, where $v$ is a node in $G$
        \EndFor
        \State $\phi_i \leftarrow$ initialization
        \While{not converged} 
            \For{each node $v \in V$ in task $t_i$}
                \State \slash* Subgraph sampling and readout by Eq.~\eqref{eq.readout} *\slash
                \State Sample $v$'s $k$-hop subgraph $S_v$
                \State $\vec{s}_{v} \leftarrow \textsc{Average}(\{\vec{h}_u \cdot \mathtt{sim}(\vec{h}_u,\vec{h}_v) : u \in V(S_v)\})$
                \State \slash* Generate pattern-based prompts by Eq.~\eqref{eq.prompt-generation} *\slash
                \State $\vec{p}_{t_i,v} \leftarrow \textsc{CondNet}(\vec{s}_{v};\phi_{t_i})$
                \State \slash* Prompt modification by Eq.~\eqref{eq.prompt} *\slash
                \State $\tilde{\vec{h}}_{t_i,v} \leftarrow \vec{p}_{t_i,v} \odot \vec{h}_{v}$
            \EndFor
            \State $\vec{h}_{t_i,G} = \textsc{Average}(\tilde{\vec{h}}_{t_i,v} : v \in \mathcal{V})$
            \State \slash* Update prototypical subgraphs *\slash
            \For{each class $c$ in task $t_i$} 
                \State ${\vec{\bar{h}}}_{t_i,c} \leftarrow \textsc{Average}(\vec{\tilde{h}}_{t_i,x}$: instance $x$ belongs to class $c$)
            \EndFor
            \State \slash* Optimizing the parameters in condition-net *\slash
            \State Calculate $\bL_\text{down}(\phi_i)$ by Eq.~\eqref{eq.prompt-loss}
            \State Update $\phi_i$ by backpropagating  $\bL_\text{down}(\phi_{t_i})$
        \EndWhile    
    \EndFor
    \State \Return $\{\phi_{t_1},\ldots,\phi_{t_n}\}$
\end{algorithmic}
\end{algorithm}

\stitle{Details of baselines.}
We use the authors' code for all baselines, if available. To ensure a fair comparison, each model is tuned while referencing the settings recommended in their respective publications. We use early stopping strategy in training and set the patience value to 50 steps. The number of training epochs is set to 2,000. 
\begin{itemize}[leftmargin=*]
    \item For GCN \cite{kipf2016semi}, we employ a 3-layer architecture on Wisconsin, Squirrel, Chameleon, Cornell datasets and a 2-layer architecture on Cora, Citeseer, ENZYMES, PROTEINS, COX2, BZR datasets.  The hidden dimension is set to 256. 
    \item For GAT \cite{velivckovic2017graph}, we employ a 2-layer architecture and set the hidden dimension to 256. Additionally, we apply 8 attention heads in the first GAT layer.
    \item For H2GCN \cite{zhu2020beyond}, we employ a 2-layer architecture and set the hidden dimension to 256. 
    \item For {FAGCN} \cite{bo2021beyond}, we employ a 2-layer architecture. The hyper-parameter settings are: eps = 0.3, dropout = 0.5, hidden = 256.  We use \textsc{relu} as the activation function. 
    \item For DGI \cite{velivckovic2017graph}, we utilize a 1-layer GCN as the base model and set the hidden dimension to 256. Additionally, we employ \textsc{prelu} as the activation function.
    \item For GraphCL \cite{you2020graph}, a 1-layer GCN is also employed as its base model, with the hidden dimension set to 256. Specifically, we select edge dropping as the augmentation, with a default augmentation ratio of 0.2.
    \item For {DSSL} \cite{xiao2022decoupled}, we search the hidden dimension in \{64, 256, 2048\}. We report the best performance on PROTEINS and ENZYMES with a hidden size of 64, Cora and Citeseer with 2048, and the rest datasets with 256.
    \item For {GraphACL} \cite{xiao2024simple}, we search the hidden dimension in \{64, 256, 1024, 2048\}. We report the best performance on PROTEINS and ENZYMES with a hidden size of 64, Cora and Citeseer with 2048, and the rest datasets with 256. 
    \item For GPPT \cite{sun2022gppt}, we utilize a 2-layer GraphSAGE as its base model, setting the hidden dimensions to 256. For the GraphSAGE backbone, we employ a mean aggregator.
    \item For GraphPrompt \cite{liu2023graphprompt}, we employ a 3-layer architecture on Wisconsin, Squirrel, Chameleon, and Cornell, and a 2-layer architecture on the rest.  Hidden dimensions are set to 256. We use link prediction as the pre-training task.
    \item For GraphPrompt+ \cite{yu2024generalized}, we employ a 2-layer GCN on Cora, Citeseer, ENZYMES, PROTEINS, COX2, and BZ, and a 3-layer GCN on the rest. Hidden dimensions are set to 256. We use link prediction as the pre-training task.
\end{itemize}

\stitle{Details of \model.}
For our proposed \model, we utilize a 2-layer FAGCN architecture as the backbone for pre-training on the non-homophilic graphs, namely, Wisconsin, Squirrel, Chameleon, and Cornell, with edge-dropping implemented on the subgraph level and hidden dimensions set to 256. For the remaining more homophilic graphs, we employ a 1-layer GCN for pre-training, with the hidden dimensions set to 256, except for PROTEINS and ENZYMES, which use hidden dimensions of 64.
We adopt a non-homophily pre-training task GraphCL \cite{you2020graph} for all datasets except for PROTEINS and ENZYMES.
Specifically, GraphCL does not work well for \model\ on the two datasets, and instead we use link prediction \cite{liu2023graphprompt} and DSSL \cite{xiao2022decoupled}, respectively. Note that the non-homophily task GraphCL still works well on most datasets including all of the non-homophilic graphs. 
For the condition-net, we set the hidden dimension to 64 for all datasets. All experiments are conducted with a random seed of 39. 

\begin{table}[tbp] 
    \centering
    \small
    \caption{Comparison of the number of tunable parameters during the downstream adaptation phase. 
    }
    \vspace{-2mm}
    \label{table.parameters-num}%
    \begin{tabular}{@{}l|rrrr@{}}
    \toprule
   Methods & {Wisconsin} & {Chameleon} & {Citeseer}& {Cora}\\\midrule
    \method{GCN} 
    & 501,504 & 660,736 & 947,968& 366,848\\     
    \method{FAGCN}
    & 440,654& 601,130& 956,654& 370,994\\
    \method{GraphCL}
    & 1,280& 1,280& 1,536& 1,792\\
    \method{GraphACL}&  1,280&  1,280& 12,288& 14,336\\
    \method{GraphPrompt}
    & 256 & 256 & 256 & 256\\
    \method{GraphPrompt+}
    & 512 & 512 & 512 & 512\\
    \method{\model}
    & 32,768 & 32,768 & 32,768 & 32,768\\\bottomrule
        \end{tabular}
\end{table}


\subsection{Parameter Efficiency}\label{app.exp.para}
We evaluate the parameter efficiency of \model\ compared to other notable methods. Specifically, we evaluate the number of parameters that need to be updated or tuned during the downstream adaptation phase, and present the results in Table~\ref{table.parameters-num}.
For GCN and FAGCN, since these models are trained end-to-end, all model weights must be updated, leading to the least parameter efficiency.
In contrast, for GraphCL and GraphACL, only the downstream classifier is fine-tuned, while the pre-trained model weights are frozen, significantly reducing the number of tunable parameters in the downstream phase.
Prompt-based methods GraphPrompt and GraphPrompt+ are the most parameter-efficient, as prompts are lightweight and typically contain fewer parameters than the downstream classifier. 
Although our conditional prompt design requires updating more tunable parameters than GraphPrompt and GraphPrompt+ during downstream adaptation, the increase is minor compared to updating the entire model weights, and thus does not pose a major issue.

\begin{figure}[t]
\centering
\includegraphics[width=1\linewidth]{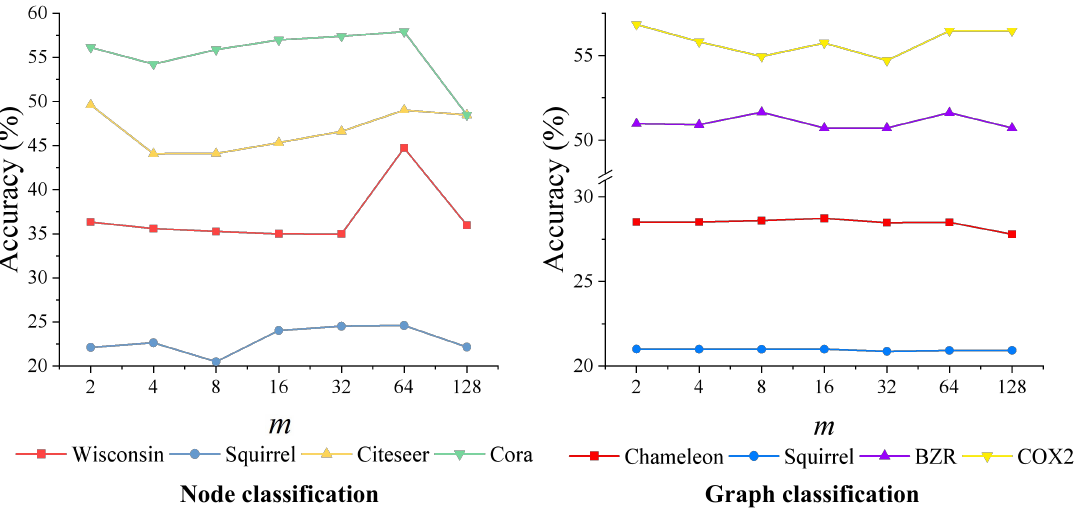}
\caption{Impact of hidden dimension $m$ in the condition-net.}
\label{fig.hyperpara}
\end{figure}

\subsection{Hyperparameter Analysis}
In our experiments, we use a 2-layer MLP with a bottleneck structure as the condition-net. We evaluate the impact of the hidden dimension $m$ of the condition-net, and report the corresponding performance in Fig.~\ref{fig.hyperpara}. 
We observe that for both node and graph classification, $m=64$ generally yields optimal performance, which we have adopted in all other experiments. Specifically, smaller values of $m$ may lack sufficient model capacity, while larger $m$ may introduce too many learnable parameters, leading to overfitting in few-shot settings.

\end{document}